\documentclass[conference]{IEEEtran}
\IEEEoverridecommandlockouts

\usepackage{cite}
\usepackage{amsmath,amssymb,amsfonts}
\usepackage{algorithmic}
\usepackage{graphicx}
\usepackage{textcomp}
\usepackage{xcolor}
\def\BibTeX{{\rm B\kern-.05em{\sc i\kern-.025em b}\kern-.08em
    T\kern-.1667em\lower.7ex\hbox{E}\kern-.125emX}}

\usepackage{booktabs,array,multirow,cite,balance,placeins}
\usepackage[hidelinks]{hyperref}
\usepackage{amsthm}

\newtheorem{proposition}{Proposition}
\newsavebox{\resulttablebox}

\begin{document}

\title{CTLite: Lightweight Irregular Multivariate Time Series Forecasting with Reusable Continuous-Time Summaries
\thanks{\textsuperscript{*}Equal contribution. \textsuperscript{$\dag$}Corresponding author.} 
\thanks{This work was supported by the Deep Space Exploration Directed Cultivation Fund of USTC.}
}

\author{
\IEEEauthorblockN{Shusen Ma\textsuperscript{*}}
\IEEEauthorblockA{
\textit{Institute of Advanced Technology} \\
\textit{University of Science and Technology of China}\\
Hefei, China
}
\and
\IEEEauthorblockN{Xudong Zuo\textsuperscript{*}}
\IEEEauthorblockA{
\textit{Institute of Advanced Technology} \\
\textit{University of Science and Technology of China}\\
Hefei, China
}
\and
\IEEEauthorblockN{Yangchun Xie}
\IEEEauthorblockA{
\textit{Institute of Advanced Technology} \\
\textit{University of Science and Technology of China}\\
Hefei, China
}
\and
\IEEEauthorblockN{Yun-Bo Zhao\textsuperscript{$\dag$}}
\IEEEauthorblockA{
\textit{Department of Automation} \\
\textit{University of Science and Technology of China}\\
Hefei, China \\
ybzhao@ustc.edu.cn
}
}

\maketitle

\begin{abstract}
Irregular multivariate time series forecasting (IMTSF) requires capturing temporal dynamics and cross-variable dependencies from sparse, asynchronous observations. Existing methods often rely on learned temporal encoders, coupling historical representation construction with model optimisation and repeatedly processing the same observations. We propose CTLite, a lightweight forecasting framework that decouples historical summarisation from learned prediction. Specifically, a parameter-free continuous-time moment encoder summarises local levels and trends at multiple temporal scales while suppressing poorly supported slope estimates. A compact readout combines damped local extrapolation with a low-rank cross-variable correction that accounts for observation freshness, target availability, and forecast horizon. Since the historical summaries are independent of trainable parameters and forecast queries, they can be reused across optimisation steps and repeated predictions for an unchanged history without altering predictions or parameter gradients. Experiments on four real-world benchmarks show that CTLite achieves the best or second-best mean MSE and MAE among the evaluated methods with only 88–2,537 trainable parameters. Under matched GPU workloads at batch size 32, CTLite reduces inference latency by 18.6–87.3\% compared with the latest models, including the cost of historical encoding. Reusing summaries further reduces total runtime by approximately 52\% over 16 forecast calls on the same history compared with recomputing the encoder for every call, including the initial encoding cost. The code is at: https://github.com/SSMa913/NICLab-CTLite.
\end{abstract}


\section{Introduction}
Irregular multivariate time series arise when variables are measured asynchronously, and observations are missing. In clinical monitoring, environmental sensing, and activity recognition, a forecast must account for both the recorded values and the time over which they were observed. Two histories with the same values and ordering can imply different dynamics when their timestamps differ; two equally long histories can provide different amounts of evidence when their observation masks differ. Effective irregular multivariate time series forecasting (IMTSF) therefore requires a representation of temporal dynamics, observation support, and dependencies across variables.

Recent methods address these requirements through continuous latent dynamics, graph interactions, and temporal patches \cite{rubanova2019,debrouwer2019,yalavarthi2024,zhang2024,hipatch2025}. Efficient architectures have also become a direct focus: APN uses time-aware aggregation with a shallow predictor, while KAFNet revisits pre-alignment for efficient irregular forecasting \cite{liu2026,kafnet2026}. These advances raise a more specific question about where learned computation is needed. \textit{Can historical information be represented by fixed, compact statistics while a small learned readout supplies the forecasting capacity?}

\begin{figure}[t]
 \centering
\includegraphics[width=\linewidth]{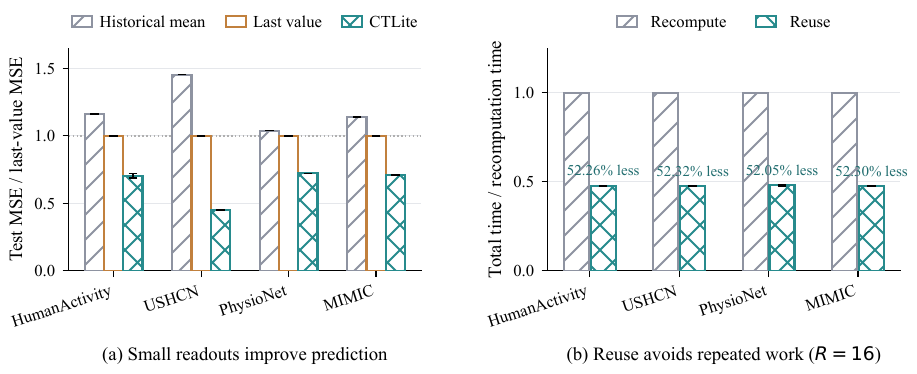}
\caption{Compact prediction and historical reuse. (a) Test MSE normalised by last-value forecasting on the same test records. (b) Time for 16 forecast calls on one unchanged batch, normalised by recomputing the history for every call. Reuse includes one initial encoding. Bars show five-seed means and sample Standard Deviation (SD).}
 \label{fig:motivation}
\end{figure}

This question matters when histories are processed repeatedly. Time-series databases such as Gorilla and Apache IoTDB provide efficient storage and access to historical measurements \cite{gorilla,iotdb}; a forecasting operator must then transform the retrieved observations into predictions. Training revisits the same windows across epochs, and inference can request several forecasts from an unchanged window. A learned encoder's cached representation becomes outdated when its parameters change. A parameter-independent representation can instead remain valid across readout updates, allowing historical aggregation to be performed once for a fixed window and preprocessing configuration. The design challenge is to obtain this reuse without reducing the history to an inadequate statistic.

Compact regular-grid forecasters suggest that a carefully chosen representation can reduce the amount of learned computation \cite{zeng2023,xu2024,lin2024,ma2026}. Irregular observations require additional structure. Array position does not measure elapsed time, a last value contains no trend information, and a local slope can be unreliable when only one observation effectively supports it. Moreover, sparsely observed variables may benefit from related variables, but a dense interaction matrix introduces a quadratic parameter cost. A useful compact representation must preserve local dynamics together with the evidence supporting them, and expose this information to an economical cross-variable predictor.

We propose \textbf{CTLite}, which separates fixed historical summarisation from learned prediction. Its continuous-time moment encoder computes weighted levels and slopes at several fixed decay rates directly from observed timestamps and masks. An effective-support factor suppresses poorly supported slopes. For each forecast query, a compact readout combines these summaries with a damped extrapolation step, whose growth saturates at a learned scale. A low-rank residual transfers recent values across variables, accounting for source freshness, target observation state, and forecast distance. Both residual branches are initialised at zero, so learning starts from the last-value predictor. The resulting model contains 88--2,537 trainable parameters across the four evaluated datasets.

Fig.~\ref{fig:motivation} illustrates two properties that motivate our design. First, a fixed historical representation with only 88–2,537 learned parameters reduces MSE by 27.82–54.71\% relative to last-value forecasting across four datasets. The gains show that a small predictor can extract useful information beyond the latest observation. Second, reusing the representation reduces the cumulative time for 16 forecast calls by 52.05–52.32\%, even when initial encoding is included. These observations motivate separating historical aggregation from learned prediction: the representation should retain useful temporal information while remaining reusable when the prediction weights change. The reuse experiment evaluates CTLite’s execution paths, not a unique advantage over every possible cached baseline.

Our contributions are summarised as follows:
\begin{itemize}
    \item We propose CTLite, a compact forecasting model that separates historical encoding from learned prediction. Its parameter-free encoder summarises irregular observations through masked moments at multiple temporal scales, while a support-aware local readout limits extrapolation from poorly observed trends. For an unchanged history and preprocessing configuration, these summaries can be reused across training updates and forecast queries without changing the model’s predictions or parameter gradients.
    \item We introduce a low-rank correction that complements local temporal predictions with information from other variables. The correction accounts for the time since each source variable’s last observation, the target variable’s observation availability, and the forecast horizon. This design incorporates observation timing and missingness into cross-variable information transfer, with an interaction parameter count that grows linearly with the number of variables for a fixed rank.
    \item We evaluate CTLite on four irregular forecasting benchmarks using five random seeds, with CTLite containing only 88–2,537 trainable parameters and achieving competitive forecasting accuracy. Matched measurements show lower full-forward latency and peak training memory. 
\end{itemize}

\begin{figure*}[t]
\centering
\raisebox{-0.0pt}{\includegraphics[width=516.01476pt]{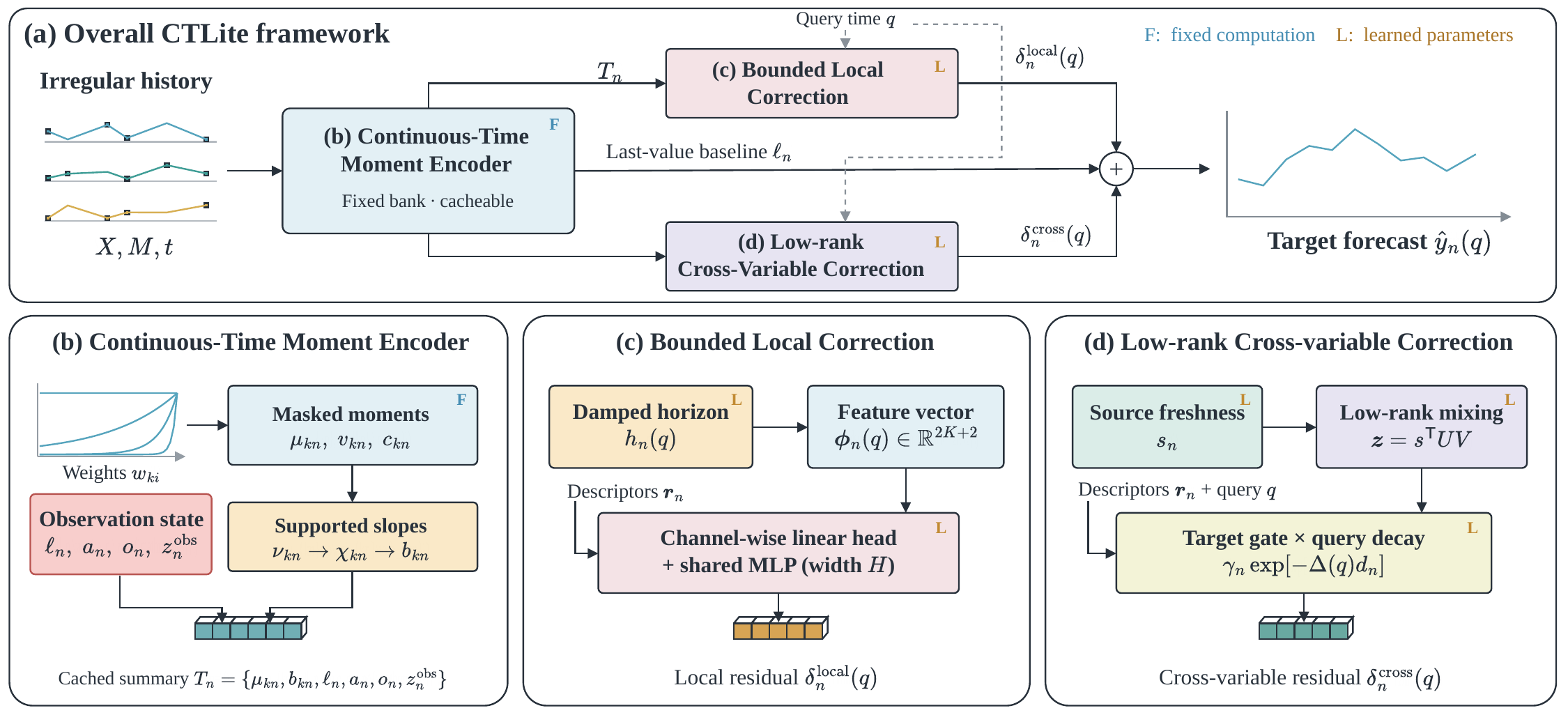}}
\caption{Overall CTLite framework. F and L denote fixed and learned computations, respectively.}
\label{fig:architecture}
\end{figure*}

\section{Related Work}
\subsection{Structured Forecasting on Regular Grids}
DLinear studies decomposition with linear prediction \cite{zeng2023}. FITS uses frequency interpolation, SparseTSF exploits cross-period sparsity, and MixLinear combines temporal and low-rank spectral modelling to obtain small predictors \cite{xu2024,lin2024,ma2026}. CycleNet explicitly models recurring patterns and predicts the residual \cite{cyclenet2024}. Other architectures impose structure through multiscale mixing in TimeMixer \cite{timemixer2024} and modern convolutional blocks in ModernTCN \cite{moderntcn2024}. Cross-variable modelling also takes several forms: PCDformer \cite{pcdformer2024} and iTransformer \cite{itransformer2024} attend over variable tokens, SOFTS aggregates and redistributes a shared series core \cite{softs2024}, and DUET uses temporal and channel clustering \cite{duet2025}. C3RL jointly leverages channel-independence and channel-mixing through a Siamese contrastive learning framework with adaptive loss weighting to improve multivariate time-series forecasting across different model architectures \cite{ma2026c3rl}. THCVformer jointly models variable heterogeneity, temporal dependencies, and inter-variable correlations to enhance prediction accuracy \cite{ma2026thcvformer}.
These methods illustrate how temporal and inter-variable structure can guide model design on regularly indexed sequences. CTLite constructs its historical features from the actual timestamps and masks of irregular observations.

\subsection{Representations for Irregular Observations}
GRU-D incorporates missingness and elapsed time into recurrent updates \cite{che2018}. Latent ODEs infer a latent initial state and model continuous dynamics \cite{rubanova2019}, whereas GRU-ODE-Bayes combines continuous evolution with updates at observations \cite{debrouwer2019}. Set-function models aggregate timestamped observations \cite{horn2020}, and mTAN learns attention-based temporal representations \cite{shukla2021}. Graph Neural Flows further couple continuous flows with graph-structured dependencies \cite{gneuralflow2024}.

Recent IMTSF methods organise observations and interactions at multiple resolutions. GraFITi represents irregular series using a sparse bipartite graph \cite{yalavarthi2024}. t-PatchGNN learns transformable temporal patches \cite{zhang2024}, and Hi-Patch builds hierarchical patch representations \cite{hipatch2025}. ASTGI learns adaptive spatio-temporal graph interactions \cite{astgi2026}, while DyH2 models heterogeneous interactions across observation and variable levels \cite{dyh2_2026}. TimeCHEAT combines local channel-dependent and global channel-independent modelling \cite{timecheat2025}. ReIMTS constructs recursive multiscale representations while preserving original timestamps \cite{reimts2026}.

Efficiency and alternative temporal representations are also active directions. APN combines time-aware aggregation with a shallow predictor \cite{liu2026}. KAFNet uses kernel-based pre-alignment and efficient frequency-domain modelling \cite{kafnet2026}. Two recent preprints explore observation-time-dependent linear maps in AiT \cite{wang2025} and learned, density-corrected basis representations in DNBNet \cite{li2026}. CTLite uses fixed exponential weights and masked moments of the observed samples. Keeping this historical computation independent of the learned readout is what permits its reuse across optimisation updates.

\subsection{Probabilistic Forecasting and Evaluation}
ProFITi models joint predictive distributions through conditional normalising flows \cite{profiti2025}, and MOSES develops marginalisation-consistent probabilistic forecasts \cite{moses2026}. CTLite produces point forecasts; its independent evaluation of query times concerns deterministic predictions rather than consistency of predictive distributions.

Recent benchmarks also broaden evaluation beyond individual model comparisons. TFB studies comprehensive and fair forecasting evaluation under unified protocols \cite{tfb2024}. Physiome-ODE provides irregular forecasting tasks derived from biological dynamical systems \cite{physiome2025}, and Time-IMM introduces irregular multimodal multivariate data \cite{timeimm2025}. PYRREGULAR provides a unified framework with classification benchmarks \cite{pyrregular2026}. These resources cover complementary tasks and data regimes. The present experiments use four established forecasting datasets and report comparisons under the stated preprocessing and execution settings.

\subsection{Historical Representation and Reuse}
Systems such as Gorilla and Apache IoTDB optimise historical data storage and access \cite{gorilla,iotdb}. CTLite operates on retrieved forecasting windows. With fixed model weights, learned historical representations can also be cached when their encoders are query-independent. CTLite's parameter-free encoder additionally allows the same stored summary to remain valid while the readout weights change. Reuse requires an unchanged history and fixed encoder configuration; the proposed implementation reconstructs the summary when either changes.

\section{Method}

Fig.~\ref{fig:architecture} illustrates CTLite and its three computational modules. Given an irregular history, the \textbf{Continuous-Time Moment Encoder} (Fig.~\ref{fig:architecture}(b)) constructs a reusable summary for each variable. For a query time, the \textbf{Bounded Local Correction} (Fig.~\ref{fig:architecture}(c)) uses the variable’s own summary to produce a local residual, while the \textbf{Low-Rank Cross-Variable Correction} (Fig.~\ref{fig:architecture}(d)) combines information across variables to produce a cross-variable residual. The two residuals are added to the Last-value baseline to obtain the Target forecast, as shown in Fig.~\ref{fig:architecture}(a).

\subsection{Task Formulation and Overall Framework}
For one record, let $X\in\mathbb R^{L\times N}$ contain historical values at timestamps $t_1,\ldots,t_L$. Here $N$ is the number of variables and $L$ is the padded history length. The mask $M\in\{0,1\}^{L\times N}$ has $M_{in}=1$ only when variable $n$ is observed at timestamp $t_i$. Missing values and padding have mask zero. Given future query times $q_1,\ldots,q_Q$, the predictor returns an estimate for each requested variable and time. The prediction interface takes only the historical values, timestamps, masks, and requested query times.

Times follow the nonnegative normalised convention of the data pipeline. The shared origin and centred historical times are
\begin{equation}
t_0=\max_{i:\sum_n M_{in}>0}t_i,\qquad u_i=t_i-t_0.
\label{eq:origin}
\end{equation}
We set $t_0=0$ if the entire history is empty. Let $\ell_n$ denote the latest observed value of variable $n$, or zero if it is unobserved. CTLite predicts
\begin{equation}
\widehat y_n(q)=\ell_n+\delta_n^{\mathrm{local}}(q)+\delta_n^{\mathrm{cross}}(q).
\label{eq:prediction}
\end{equation}
The local residual ($\delta_n^{\mathrm{local}}$) uses the variable's own history. The cross-variable residual ($\delta_n^{\mathrm{cross}}$) uses recent information from all variables. The origin is determined solely by the historical observations; query independence is established in the reuse analysis.

\subsection{Continuous-Time Moment Encoder}
\label{sec:moments}
\subsubsection{Weights}
Choose $K$ fixed nonnegative decay rates $\Lambda=\{\lambda_1,\ldots,\lambda_K\}$, including zero. Their dataset-specific values are reported in Table~\ref{tab:settings}. At scale $k$, observation $i$ receives weight:
\begin{equation}
w_{ki}=\exp\!\left(\operatorname{clip}(\lambda_k u_i,-60,0)\right).
\label{eq:weights}
\end{equation}
The function $\operatorname{clip}(z,a,b)=\min\{b,\max\{a,z\}\}$ limits its argument to $[a,b]$. A zero rate gives equal weights, while a larger rate emphasises recent observations. Rates are measured relative to normalised time. We clip the exponent below at $-60$ for numerical stability.

\subsubsection{Masked moments}
For a scalar sequence $f_i$, define
\begin{equation}
\begin{aligned}
D_{kn} & =\max\!\left\{\sum_i w_{ki}M_{in},\epsilon\right\},\\
\mathcal E_{kn}[f] & =D_{kn}^{-1}\sum_i w_{ki}M_{in}f_i, \\
\epsilon & =10^{-6} .
\end{aligned}
\label{eq:average}
\end{equation}
The index $n$ determines which observations contribute. The denominator floor $\epsilon$ prevents division by zero. Masked values are set to zero before moment accumulation. When the floor is inactive, $\mathcal E_{kn}$ is a normalised weighted average.

The level is $\mu_{kn}=\mathcal E_{kn}[x_n]$, where $x_n=(x_{1n},\ldots,x_{Ln})$. With mean time $\bar u_{kn}=\mathcal E_{kn}[u]$, define temporal variance and time--value covariance as
\begin{equation}
\begin{aligned}
v_{kn}&=\max\{\mathcal E_{kn}[u^2]-\bar u_{kn}^2,0\},\\
c_{kn}&=\mathcal E_{kn}[u x_n]-\bar u_{kn}\mu_{kn}.
\end{aligned}
\label{eq:covariance}
\end{equation}
The nonnegative clamp removes small negative variances due to numerical cancellation. Actual time differences enter both the weights and the moments. The summary therefore changes when the observation times change, even if their order and values do not.

\subsubsection{Supported slopes}
A local slope estimated from very few effective observations can be unreliable. We use a stabilized effective count $\nu_{kn}$ and support factor $\chi_{kn}$:
\begin{equation}
\begin{aligned}
\nu_{kn}&=\frac{D_{kn}^2}{\max\{\sum_i w_{ki}^2M_{in},\epsilon\}},\\
\chi_{kn}&=\frac{\max\{\nu_{kn}-1,0\}}{\nu_{kn}+1},\\
b_{kn}&=\chi_{kn}\frac{c_{kn}}{v_{kn}+\rho},\qquad \rho=10^{-3}.
\end{aligned}
\label{eq:slope}
\end{equation}
The ridge coefficient $\rho$ limits sensitivity to small temporal variance. With inactive floors and normalized weights $p_i=w_{ki}M_{in}/D_{kn}$, the effective count is $1/\sum_i p_i^2$. Equal weights on $m$ samples give $\nu_{kn}=m$, whereas a dominant single sample gives a value near one. Consequently, $\chi_{kn}$ suppresses slopes with little support. It shrinks slopes according to their effective observation support. Empty and singleton histories have zero slope, including cases in which a numerical floor is active.

\subsubsection{Observation state}
For each variable, let $o_n=\sum_i M_{in}$ be its observation count and $a_n$ its age at $t_0$. If variable $n$ has a latest observation at time $t_n^*$, then $a_n=\operatorname{clip}(t_0-t_n^*,0,1)$. Otherwise, set $a_n=1$ and $\ell_n=0$. The availability indicator is $z_n^{\mathrm{obs}}=\mathbf1[o_n>0]$, where $\mathbf1$ is the indicator function. Define the descriptor:
\begin{equation}
\boldsymbol r_n=[a_n,(o_n+1)^{-1},z_n^{\mathrm{obs}}]^{\mathsf T}.
\label{eq:descriptor}
\end{equation}
Its three entries describe staleness, sparsity, and availability. 

\subsubsection{Cached summary}
The encoder outputs $T_n=\{\mu_{kn},b_{kn}\}_{k=1}^K\cup\{\ell_n,a_n,o_n,z_n^{\mathrm{obs}}\}$ for each variable and stores $t_0$ once. No part of this summary depends on a learned parameter or on the forecast queries.

\subsection{Bounded Local Correction}
\label{sec:local}

\subsubsection{Damped horizon}
The Damped horizon block converts the elapsed time into a bounded extrapolation step $h_n(q)$, which controls how strongly the historical slopes contribute to the prediction.

Specifically, for a query time $q$, define the nonnegative forecast distance $\Delta(q)=\max\{q-t_0,0\}$ and represented elapsed time $e_n(q)=\Delta(q)+a_n$. A learned positive coefficient $d_n=\operatorname{softplus}(\theta_n)$ controls the extrapolation step
\begin{equation}
h_n(q)=\frac{1-\exp[-d_ne_n(q)]}{d_n}.
\label{eq:horizon}
\end{equation}
Here $\theta_n\in\mathbb R$ is trainable and $\operatorname{softplus}(x)=\log(1+e^x)$. For small $d_ne_n$, $h_n$ approximates elapsed time. For a long horizon, it approaches $1/d_n$. Thus, a poorly estimated slope is not multiplied by an arbitrarily large forecast distance. 

\subsubsection{Feature vector}
The local feature vector has $F=2K+2$ entries:
\begin{equation}
\boldsymbol\phi_n(q)=\begin{bmatrix}
\{\mu_{kn}-\ell_n\}_{k=1}^K\\
\{b_{kn}h_n(q)\}_{k=1}^K\\
\ell_n\\
\ell_n\big(\exp[-d_ne_n(q)]-1\big)
\end{bmatrix}.
\label{eq:features}
\end{equation}
The first group describes historical levels relative to the latest value. The second describes supported trend increments. The final two entries retain the latest value and its damped change. They allow a small readout to combine level correction, trend extrapolation, and attenuation.

\subsubsection{Channel-wise linear head + shared MLP (width $H$) }
The residual combines a variable-specific linear map with a shared nonlinear map:
\begin{equation}
\delta_n^{\mathrm{local}}(q)=\boldsymbol\alpha_n^{\mathsf T}\boldsymbol\phi_n(q)+\beta_n
+g_\psi([\boldsymbol\phi_n(q),\boldsymbol r_n]).
\label{eq:local}
\end{equation}
The coefficients are $\boldsymbol\alpha_n\in\mathbb R^F$ and $\beta_n\in\mathbb R$. Square brackets concatenate vectors. The shared function is
$g_\psi(\xi)=w_2^{\mathsf T}\tanh(W_1\xi+b_1)+b_2$,
where $W_1\in\mathbb R^{H\times(F+3)}$, $b_1,w_2\in\mathbb R^H$, and $b_2\in\mathbb R$. The hidden width $H$ is small, and the same network is used for every variable. This sharing avoids allocating a nonlinear network to each channel. We initialise $\boldsymbol\alpha_n$, $\beta_n$, $w_2$, and $b_2$ to zero, so the initial local residual is zero.

\subsection{Low-Rank Cross-Variable Correction}
\label{sec:cross}
A variable's own summary may be insufficient when its measurements are sparse or stale. We form a source vector $s\in\mathbb R^N$ from recent values and mix it through a rank-$r$ factorisation:
\begin{equation}
s_n=\ell_n z_n^{\mathrm{obs}}\exp(-a_nd_n),\qquad
\boldsymbol z=s^{\mathsf T}UV.
\label{eq:mixing}
\end{equation}
The learned matrices are $U\in\mathbb R^{N\times r}$ and $V\in\mathbb R^{r\times N}$. Their product uses $2Nr$ parameters instead of $N^2$. The integer $r$ is distinct from the descriptor vector $\boldsymbol r_n$. The availability indicator removes unobserved sources, and the exponential factor reduces the contribution of older sources. Self-contributions are allowed because the diagonal of $UV$ is not removed.

A target-specific gate then controls the context contribution:
\begin{equation}
\begin{aligned}
\gamma_n&=\sigma(\boldsymbol\eta_n^{\mathsf T}\boldsymbol r_n),\\
\delta_n^{\mathrm{cross}}(q)&=\gamma_n\exp[-\Delta(q)d_n]z_n.
\end{aligned}
\label{eq:cross}
\end{equation}
Here $\boldsymbol\eta_n\in\mathbb R^3$ is learned and $\sigma(x)=1/(1+e^{-x})$. Source attenuation depends on the age of each input variable. The gate depends on the observation state of the target variable. The final decay depends on forecast distance. The learned gate is not constrained to increase monotonically with missingness.

We initialise $U$ from a zero-mean Gaussian with standard deviation $1/\sqrt N$ and set $V=0$. The initial cross residual is therefore zero. Together with the local initialisation, the complete model starts from persistence, $\widehat y_n(q)=\ell_n$, and learns only the corrections supported by training.

\subsection{Training, Reuse, and Computational Properties}
\subsubsection{Objective}
For a minibatch $\mathcal B$, let $A_{bjn}$ indicate whether the target of record $b$, query $j$, and variable $n$ is observed. Training minimizes
\begin{equation}
\mathcal L_{\mathcal B}=
\frac{\sum_{b\in\mathcal B,j,n}A_{bjn}(\widehat y_{bjn}-y_{bjn})^2}
{\max\{1,\sum_{b\in\mathcal B,j,n}A_{bjn}\}}.
\label{eq:loss}
\end{equation}
All four datasets use Adam with a fixed learning rate, gradient-norm clipping at 1, at most 200 epochs, and early stopping after 30 epochs without validation-MSE improvement. 

\subsubsection{Exact reuse of the historical summary}
Write the model as $f_\theta(X,M,t,q)=D_\theta(E(X,M,t),q)$, where $E$ is the fixed encoder and $D_\theta$ the learned readout. For an unchanged history, replacing $E(X,M,t)$ by its stored value leaves both the prediction and the gradient with respect to $\theta$ unchanged. The encoder has no parameter gradient that would otherwise need to be updated. Learned damping, gates, and query features remain in $D_\theta$ and are recomputed at each step.

This reuse is valid only while the history, normalisation, decay bank, and numerical constants are unchanged. A new observation or a changed scale bank requires a new summary. The current implementation does not maintain an exact streaming update across a moving origin. For multiple queries against the same window, however, the history is encoded once, and each query is decoded independently. Adding or removing another query cannot change a prediction already requested.

\subsubsection{Bounds and size}
For $x\geq0$, $0\leq1-e^{-x}\leq\min\{x,1\}$. Substitution into Eq.~\eqref{eq:horizon} gives
\begin{equation}
0\leq h_n(q)\leq\min\{e_n(q),1/d_n\}.
\label{eq:bound}
\end{equation}
For fixed $d_n>0$, Eq.~\eqref{eq:bound} bounds the extrapolation step $h_n$. Conditional on masks, times, and learned parameters, Eq.~\eqref{eq:cross} has the form $\ell^{\mathsf T}D_sUV D_q$, where $D_s$ and $D_q$ are diagonal source and target factors. Its rank is at most $r$. The rank statement concerns the cross residual only, not the complete nonlinear predictor.

The linear heads, biases, damping coefficients, and gates contain $N(2K+7)$ parameters. The factorisation adds $2Nr$ and the shared MLP adds $H(2K+7)+1$. Hence
\begin{equation}
P=N(2K+7+2r)+H(2K+7)+1.
\label{eq:parameters}
\end{equation}
This count is independent of $L$ and $Q$. For batch size $B$, encoding takes $O(BLNK)$ work, the local readout takes $O(BQNKH)$, and cross-variable mixing takes $O(BNr+BQN)$. The summary stores $B[N(2K+4)+1]$ scalar values. These bounds make clear which costs depend on the observed history and which remain for each query. Full inference still includes encoding whenever a fresh history is supplied.

\subsection{Why the Summary Is Stable and Reusable}
\label{sec:properties}
The representation is deliberately restricted. Its usefulness depends on retaining a supported local trend, controlling extrapolation, and separating fixed historical work from learned computation. We establish the properties needed for those choices below. 

\subsubsection{Connection to weighted local regression}
Fix a variable and a scale, and temporarily omit their subscripts. If the numerical floors are inactive, the normalised weights are $p_i=w_iM_i/\sum_j w_jM_j$. Consider the local line with intercept $a$ and slope $b$:
\begin{equation}
\min_{a,b}\ \sum_i p_i(x_i-a-bu_i)^2+\rho b^2.
\label{eq:ridge_objective}
\end{equation}
Its intercept condition gives $a=\mu-b\bar u$. Substitution into the slope condition gives $b_*=c/(v+\rho)$. Thus Eq.~\eqref{eq:slope} is a ridge slope followed by the support-dependent shrinkage $\chi$. The shrinkage is an additional modelling choice, rather than the exact optimiser of Eq.~\eqref{eq:ridge_objective}. When denominator floors are active, $\mathcal E$ is not a probability average, and this regression interpretation does not hold exactly.

For $m$ observed entries with positive weights, Cauchy--Schwarz gives $1/m\leq\sum_i p_i^2\leq1$. Therefore $1\leq\nu\leq m$ and $0\leq\chi<1$. Equal weights recover the observation count, whereas concentration on one observation reduces the effective count toward one. For example, normalized weights $(1/2,1/2)$ give $\nu=2$ and $\chi=1/3$. Weights $(0.99,0.01)$ give $\nu\approx1.02$ and $\chi\approx0.01$. Two recorded points therefore need not provide two equally informative supports for a slope.

\begin{proposition}[Empty and singleton histories]
In real arithmetic, the stabilised moment encoder assigns zero slope to an empty history and to a history with one observed value.
\end{proposition}
\begin{proof}
An empty history has zero covariance. For a singleton of weight $w$, if $w^2\geq\epsilon$, the floors are inactive, covariance is zero, and $\nu=1$. If $\epsilon\leq w<\sqrt\epsilon$, then $D=w$ and $\nu=w^2/\epsilon<1$, so $\chi=0$. If $w<\epsilon$, then $D=\epsilon$ and $\nu=\epsilon<1$, again giving $\chi=0$. The boundary cases give the same slope.
\end{proof}
This property prevents the model from interpreting one observation as evidence for a local trend. It does not force the entire prediction to be zero for a missing variable. Learned biases and information from other variables can still contribute.

\subsubsection{Extrapolation and conditional interaction rank}
For fixed $d>0$, differentiation of $h(e)=(1-e^{-de})/d$ yields $h'(e)=e^{-de}>0$ and $h''(e)=-de^{-de}<0$. The effective step increases with elapsed time, but its marginal increase decreases. Together with Eq.~\eqref{eq:bound}, this explains how the model retains short-horizon trend information without linearly extending the same slope indefinitely. The bound is conditional on $d$. As $d$ approaches zero, its asymptotic limit $1/d$ increases. The levels, learned readout weights, and complete forecast remain unbounded.

The diagonal factors in the cross residual are
\begin{equation}
\begin{aligned}
D_s&=\operatorname{diag}\big(z_n^{\mathrm{obs}}e^{-a_nd_n}\big),\\
D_q&=\operatorname{diag}\big(\gamma_n e^{-\Delta d_n}\big).
\end{aligned}
\label{eq:diagonal_factors}
\end{equation}
Here $\operatorname{diag}$ places its arguments on the diagonal. Since multiplication cannot increase rank, $\operatorname{rank}(D_sUV D_q)\leq r$. Missing source variables contribute zero through $D_s$, while each target's gate can modulate the transferred value independently. The statement holds conditional on the observation state. It does not restrict the rank of the complete predictor across histories with different masks or times.

\subsubsection{Exact reuse during learning}
\begin{proposition}[Prediction and gradient preservation]
Let $E$ be a deterministic encoder independent of learned parameters $\theta$. For a fixed history $x$, replacing $E(x)$ in $D_\theta(E(x),q)$ by its stored value preserves predictions and derivatives with respect to $\theta$ wherever those derivatives exist.
\end{proposition}
\begin{proof}
Set $z=E(x)$. Both evaluations compute $D_\theta(z,q)$ at the same argument. Since $\partial E(x)/\partial\theta=0$, the chain-rule contribution through the encoder is zero. Both derivatives are the derivative of the same readout.
\end{proof}
Consequently, with identical initialisation, minibatch order, numerical execution, and optimiser state, replacing recomputation by exact stored summaries leaves the sequence of Adam updates unchanged. This claim concerns gradients with respect to model parameters. It does not preserve an input gradient through a detached cache. Changing preprocessing, augmenting the historical values, or learning the decay rates would invalidate the stated reuse rule.

\subsubsection{Execution and storage}
The operator has three stages. First, replace masked values with zero, find the historical origin, and accumulate the moments and observation state. Second, for each query, compute damping, local features, and the shared readout. Third, compute the source mixture once for the history, apply the target gates and query decay, and add both residuals to the latest value. No query-axis aggregation is required, so query chunks can be evaluated independently while sharing the same summary.

A stored summary occupies $4[N(2K+4)+1]$ bytes per record in FP32. Under the final configurations, this is 676 bytes for HumanActivity, 164 for USHCN, 1,732 for PhysioNet, and 5,380 for MIMIC. These values exclude the targets, decoder activations, and storage-system overhead. Retaining all summaries trades memory for repeated encoder work, so materialisation is useful only when a history is revisited enough times to justify its construction and storage.

\begin{table*}[h]
\centering
\caption{Forecasting performance on four IMTS datasets (mean $\pm$ standard deviation). Bold and underline mark the best and second-best means. $\dagger$ denotes repaired PhysioNet.}
\label{tab:accuracy}

\fontsize{10}{8}\selectfont
\setlength{\tabcolsep}{2pt}
\renewcommand{\arraystretch}{2.2}

\resizebox{\textwidth}{!}{%
\begin{tabular}{
@{}l
!{\color{black!40}\vrule width .35pt}cc
!{\color{black!40}\vrule width .35pt}cc
!{\color{black!40}\vrule width .35pt}cc
!{\color{black!40}\vrule width .35pt}cc@{}
}
\toprule
\bfseries Dataset
& \multicolumn{2}{c}{\bfseries HumanActivity}
& \multicolumn{2}{c}{\bfseries USHCN}
& \multicolumn{2}{c}{\bfseries PhysioNet $^\dagger$}
& \multicolumn{2}{c}{\bfseries MIMIC} \\
\midrule
\bfseries Metric
& \bfseries MSE $\downarrow$
& \bfseries MAE $\downarrow$
& \bfseries MSE $\downarrow$
& \bfseries MAE $\downarrow$
& \bfseries MSE $\downarrow$
& \bfseries MAE $\downarrow$
& \bfseries MSE $\downarrow$
& \bfseries MAE $\downarrow$ \\
\midrule
PrimeNet
& 4.2507$\pm$0.0041
& 1.7018$\pm$0.0011
& 0.4930$\pm$0.0015
& 0.4954$\pm$\textbf{0.0018}
& 0.8128$\pm$\textbf{0.00004}
& 0.6977$\pm$\textbf{0.00006}
& 0.9073$\pm$\textbf{0.0001}
& 0.6614$\pm$\textbf{0.0001} \\

NeuralFlows
& 0.1722$\pm$0.0090
& 0.3150$\pm$0.0094
& 0.2087$\pm$0.0258
& 0.3157$\pm$0.0187
& 0.3910$\pm$0.00303
& 0.4360$\pm$0.00201
& 0.6085$\pm$0.0101
& 0.5306$\pm$0.0066 \\

CRU
& 0.1387$\pm$0.0073
& 0.2607$\pm$0.0092
& 0.2168$\pm$0.0162
& 0.3180$\pm$0.0248
& 0.5917$\pm$0.00340
& 0.5604$\pm$0.00199
& 0.5895$\pm$0.0092
& 0.5151$\pm$0.0048 \\

mTAN
& 0.0993$\pm$0.0026
& 0.2219$\pm$0.0047
& 0.5561$\pm$0.2020
& 0.5015$\pm$0.0968
& 0.3711$\pm$0.00459
& 0.4231$\pm$0.00356
& 0.9408$\pm$0.1126
& 0.6755$\pm$0.0459 \\

SeFT
& 1.3786$\pm$0.0024
& 0.9762$\pm$0.0007
& 0.3345$\pm$0.0022
& 0.4083$\pm$0.0084
& 0.7912$\pm$0.00223
& 0.6860$\pm$0.00079
& 0.9230$\pm$0.0015
& 0.6628$\pm$\underline{0.0008} \\

GNeuralFlow
& 0.3936$\pm$0.1585
& 0.4541$\pm$0.0841
& 0.2205$\pm$0.0421
& 0.3286$\pm$0.0412
& 0.8433$\pm$0.15458
& 0.6791$\pm$0.06523
& 0.8957$\pm$0.0209
& 0.6450$\pm$0.0072 \\

GRU-D
& 0.1893$\pm$0.0627
& 0.3253$\pm$0.0485
& 0.2097$\pm$0.0493
& 0.3045$\pm$0.0305
& 0.3270$\pm$0.00157
& 0.3917$\pm$\underline{0.00091}
& 0.4759$\pm$0.0100
& 0.4526$\pm$0.0055 \\

Raindrop
& 0.0916$\pm$0.0072
& 0.2114$\pm$0.0072
& 0.2035$\pm$0.0336
& 0.3029$\pm$0.0264
& 0.3330$\pm$0.00157
& 0.3965$\pm$0.00138
& 0.6754$\pm$0.1829
& 0.5444$\pm$0.0868 \\

Warpformer
& 0.0449$\pm$0.0010
& 0.1228$\pm$0.0018
& 0.1888$\pm$0.0598
& 0.2939$\pm$0.0591
& 0.2888$\pm$0.00141
& 0.3556$\pm$0.00095
& 0.4302$\pm$0.0035
& 0.4025$\pm$0.0014 \\

tPatchGNN
& 0.0443$\pm$0.0009
& 0.1247$\pm$0.0031
& 0.1885$\pm$0.0403
& 0.3084$\pm$0.0479
& 0.2923$\pm$0.00291
& 0.3595$\pm$0.00321
& 0.4431$\pm$0.0115
& 0.4077$\pm$0.0088 \\

GraFITi
& 0.0437$\pm$\underline{0.0005}
& 0.1221$\pm$0.0017
& 0.1691$\pm$\underline{0.0093}
& 0.2777$\pm$0.0248
& \textbf{0.2873}$\pm$0.00213
& \underline{0.3531}$\pm$0.00214
& 0.4359$\pm$0.0455
& 0.4142$\pm$0.0297 \\

APN
& \textbf{0.0421$\pm$0.0001}
& \textbf{0.1159}$\pm$\underline{0.0006}
& \textbf{0.1590}$\pm$0.0137
& \textbf{0.2611}$\pm$0.0167
& 0.2947$\pm$0.00310
& 0.3610$\pm$0.00610
& \underline{0.4292}$\pm$0.0027
& \underline{0.4016}$\pm$0.0016 \\
\midrule

\textbf{CTLite}
& \underline{0.0426}$\pm$0.0010
& \underline{0.1167}$\pm$\textbf{0.0005}
& \underline{0.1648}$\pm$\textbf{0.0008}
& \underline{0.2707}$\pm$\underline{0.0024}
& \underline{0.2876}$\pm$\underline{0.00017}
& \textbf{0.3527}$\pm$0.00096
& \textbf{0.4048}$\pm$\underline{0.0008}
& \textbf{0.3844}$\pm$0.0009 \\
\bottomrule
\end{tabular}%
}
\end{table*}

\section{Experiments}
We evaluate CTLite on four IMTS benchmarks: HumanActivity, USHCN, PhysioNet, and MIMIC. Our experiments examine whether compact, reusable historical summaries can support accurate forecasting at a low computational cost. We first describe the datasets, evaluation protocol, and implementation settings, then compare forecasting accuracy and computational efficiency. We further investigate how inference and training costs change with batch size, quantify the benefit of reusing historical summaries, and assess the contribution of each modelling component through ablation studies. We report CTLite's results over five random seeds to characterise performance variability.

\begin{table}[!t]
\centering
\caption{Split sizes and observed target counts.}
\label{tab:datafull}

\small
\setlength{\tabcolsep}{4pt}
\renewcommand{\arraystretch}{1.14}

\resizebox{0.48\textwidth}{!}{%
\begin{tabular}{@{}l!{\color{black!40}\vrule width .35pt}rrrrrrr@{}}
\toprule
\bfseries Dataset
& \bfseries $N$
& \multicolumn{3}{c}{\bfseries Records (Samples)}
& \multicolumn{3}{c}{\bfseries Observed targets} \\
\cmidrule(lr){3-5}
\cmidrule(lr){6-8}
& 
& \bfseries Train
& \bfseries Val
& \bfseries Test
& \bfseries Train
& \bfseries Val
& \bfseries Test \\
\midrule
HumanActivity & 12 & 949 & 193 & 218 & 25,872 & 5,205 & 5,910 \\
USHCN & 5 & 901 & 101 & 112 & 2,964 & 334 & 368 \\
PhysioNet & 36 & 9,703 & 1,079 & 1,199 & 217,124 & 24,061 & 27,008 \\
MIMIC & 96 & 17,212 & 1,913 & 2,125 & 187,419 & 20,874 & 23,015 \\
\bottomrule
\end{tabular}
}
\end{table}

\begin{table*}[!t]
\centering
\caption{Matched efficiency at batch size 32 (mean $\pm$ standard deviation). Throughput is in thousands of records/s, and memory is peak allocated MiB. Bold is best, underline second best.}
\label{tab:efficiency}

\small
\setlength{\tabcolsep}{12pt}
\renewcommand{\arraystretch}{1.14}

\resizebox{\textwidth}{!}{%
\begin{tabular}{@{}c!{\color{black!40}\vrule width .35pt}lrrrrrr@{}}
\toprule
\bfseries Dataset
& \bfseries Model
& \bfseries Params $\downarrow$
& \bfseries Train (ms) $\downarrow$
& \bfseries Infer (ms) $\downarrow$
& \bfseries Throughput $\uparrow$
& \bfseries Train MiB $\downarrow$
& \bfseries Infer MiB $\downarrow$ \\
\midrule
\multirow{3}{*}{HumanActivity} & APN
& \underline{28,789}
& 8.905$\pm$0.034
& \underline{3.530$\pm$0.021}
& \underline{9.065$\pm$0.055}
& 463.83$\pm$1.00
& 229.90$\pm$0.41 \\

& GraFITi
& 249,409
& \underline{8.887$\pm$0.053}
& 4.205$\pm$0.007
& 7.610$\pm$0.012
& \underline{280.21$\pm$1.08}
& \underline{62.92$\pm$0.31} \\

& \textbf{CTLite}
& \textbf{365}
& \textbf{2.232$\pm$0.017}
& \textbf{0.665$\pm$0.008}
& \textbf{48.119$\pm$0.579}
& \textbf{18.46$\pm$0.04}
& \textbf{10.24$\pm$0.03} \\
\midrule
\multirow{3}{*}{USHCN} & APN
& \underline{1,970}
& \underline{3.440$\pm$0.024}
& \underline{1.062$\pm$0.006}
& \underline{30.119$\pm$0.175}
& \underline{176.30$\pm$2.71}
& \underline{69.54$\pm$1.04} \\

& GraFITi
& 495,617
& 8.516$\pm$0.114
& 3.940$\pm$0.055
& 8.122$\pm$0.113
& 339.13$\pm$8.80
& 137.81$\pm$3.01 \\

& \textbf{CTLite}
& \textbf{88}
& \textbf{2.205$\pm$0.016}
& \textbf{0.654$\pm$0.003}
& \textbf{48.926$\pm$0.250}
& \textbf{17.70$\pm$0.01}
& \textbf{9.55$\pm$0.01} \\
\midrule
\multirow{3}{*}{PhysioNet} & APN
& \underline{6,701}
& \underline{2.692$\pm$0.020}
& \underline{0.813$\pm$0.002}
& \underline{39.351$\pm$0.085}
& \underline{59.84$\pm$0.00}
& \underline{31.20$\pm$0.00} \\

& GraFITi
& 250,945
& 9.531$\pm$0.621
& 4.778$\pm$0.411
& 6.736$\pm$0.565
& 325.74$\pm$36.61
& 59.46$\pm$6.36 \\

& \textbf{CTLite}
& \textbf{889}
& \textbf{2.223$\pm$0.035}
& \textbf{0.662$\pm$0.007}
& \textbf{48.346$\pm$0.480}
& \textbf{18.25$\pm$0.00}
& \textbf{10.05$\pm$0.00} \\
\midrule
\multirow{3}{*}{MIMIC} & APN
& \underline{13,777}
& \underline{7.749$\pm$0.449}
& \underline{3.028$\pm$0.157}
& \underline{10.593$\pm$0.579}
& 387.14$\pm$48.19
& 181.86$\pm$15.15 \\

& GraFITi
& 254,785
& 10.808$\pm$1.399
& 5.241$\pm$0.924
& 6.311$\pm$1.456
& \underline{363.21$\pm$76.24}
& \underline{78.28$\pm$13.94} \\

& \textbf{CTLite}
& \textbf{2,537}
& \textbf{2.238$\pm$0.018}
& \textbf{0.667$\pm$0.006}
& \textbf{47.978$\pm$0.443}
& \textbf{20.99$\pm$0.27}
& \textbf{12.78$\pm$0.35} \\
\bottomrule
\end{tabular}%
}
\end{table*}

\subsection{Experimental Setup}
\subsubsection{Datasets and evaluation}
We use HumanActivity, USHCN, repaired PhysioNet, and MIMIC, with 12, 5, 36, and 96 variables. Clinical data originate from the PhysioNet 2012 challenge and MIMIC \cite{silva2012,johnson2016}. Table~\ref{tab:datafull} lists the supplied split sizes. We retain their normalisation, history windows, query timestamps, and masks. In particular, repaired PhysioNet uses history in $[0,36)$ hours and the first three retained future hourly bins within $[36,48)$. An observed target is a scalar value, not an entire query timestamp.

Let $\Omega_S$ be the observed scalar targets in split $S$, and let $T_S=|\Omega_S|$. The reported metrics are
\begin{equation}
\begin{aligned}
\operatorname{MSE}_S&=T_S^{-1}\!\sum_{(b,j,n)\in\Omega_S}(\widehat y_{bjn}-y_{bjn})^2,\\
\operatorname{MAE}_S&=T_S^{-1}\!\sum_{(b,j,n)\in\Omega_S}|\widehat y_{bjn}-y_{bjn}|.
\end{aligned}
\label{eq:split_metrics}
\end{equation} 
We aggregate errors over all observed scalar targets, including incomplete batches, and report five-seed (2024--2028) means and sample standard deviations: For seed errors $e_1,\ldots,e_5$, we report $\bar e=\sum_s e_s/5$ and sample deviation $s_e=[\sum_s(e_s-\bar e)^2/4]^{1/2}$.

\subsubsection{Final configurations and selection}
Table~\ref{tab:settings} gives all dataset-specific model and training settings. The numerical constants $\epsilon=10^{-6}$ and $\rho=10^{-3}$ are shared. Dropout is not used. The weight decay is coupled Adam regularisation. Each dataset uses one configuration for all five seeds.

\begin{table*}[!t]
\centering
\caption{Final model and training settings. All tasks use constant learning rates and raw Adam weights.}
\label{tab:settings}

\small
\setlength{\tabcolsep}{12pt}
\renewcommand{\arraystretch}{1.14}

\begin{tabular}{@{}l!{\color{black!40}\vrule width .35pt}lcccccc@{}}
\toprule
\bfseries Dataset
& \bfseries Fixed rates $\Lambda$
& \bfseries $r$
& \bfseries $H$
& \bfseries Learning rate
& \bfseries Batch size
& \bfseries Weight decay
& \bfseries Params \\
\midrule
HumanActivity & $\{0,1,4,16,64\}$ & 1 & 8 & 0.01 & 16 & 0.0001 & 365 \\
USHCN & $\{0,8\}$ & 1 & 2 & 0.0003 & 4 & 0.01 & 88 \\
PhysioNet  & $\{0,2,8,32\}$ & 4 & 4 & 0.001 & 64 & 0.0003 & 889 \\
MIMIC & $\{0,2,8,32,128\}$ & 4 & 8 & 0.001 & 32 & 0.0001 & 2,537 \\
\bottomrule
\end{tabular}
\end{table*}

\subsubsection{Baselines and Data Preparation}
To evaluate the performance of CTLite, we compare it with 12 baseline models. These baseline models can be broadly categorised into two groups: (1) IMTS Classification/Imputation Models, including PrimeNet \cite{chowdhury2023primenet}, SeFT \cite{horn2020}, mTAN \cite{shukla2021}, GRU-D \cite{che2018}, and Warpformer \cite{zhang2023warpformer}. (2) IMTS Forecasting Models, including NeuralFlows \cite{bilovs2021neural}, CRU \cite{schirmer2022modeling}, GNeuralFlow \cite{gneuralflow2024}, tPatchGNN \cite{zhang2024}, GraFITi \cite{yalavarthi2024}, and APN \cite{liu2026}.
Baseline results for HumanActivity, USHCN, and MIMIC are adopted from \cite{liu2026}. For PhysioNet, we use a corrected version that reconstructs hourly observations from raw records, resolves variable aliases, treats negative values as missing, and replaces the legacy 35.5-hour cutoff with a history interval
of $[0,36)$ hours. Prediction targets are the first three retained future hourly bins within $[36,48)$. Patient partitions remain unchanged, and normalisation uses training statistics only. The correction changes the observed test targets from 23,694 to 27,008, making errors across versions non-comparable.
We therefore use the corrected PhysioNet baseline results alongside CTLite evaluated on the corrected caches.

\subsection{Predictive Accuracy}
Forecasting performance on four IMTS datasets (mean $\pm$ standard deviation) is shown in TABLE \ref{tab:accuracy}. CTLite has the lowest mean MSE and MAE on MIMIC among the reported models. The errors are 0.40476 and 0.38438, improving on the best published means by 5.69\% and 4.29\%. On repaired PhysioNet, its MSE of 0.28762 is second to GraFITi's 0.2873, while its MAE of 0.35265 is the lowest reported mean. 

On HumanActivity and USHCN, CTLite ranks second in both metrics. Its MSE is 1.09\% and 3.66\% above the best published mean, respectively. It shows that a very small readout can be competitive while preserving a clear cost advantage.

\subsubsection{Variation across training seeds}
Table~\ref{tab:individual} gives every validation and test result of the complete model. Validation determines the checkpoint within a run, and test values are recorded only for that checkpoint. This exposes the variation behind the aggregate scores. Low variation across these runs indicates stability to the selected training seeds on this split.

\begin{table}[!t]
\centering
\caption{Validation and test errors of the complete model by seed.}
\label{tab:individual}
\small
\setlength{\tabcolsep}{3pt}
\renewcommand{\arraystretch}{1.14}
\resizebox{\columnwidth}{!}{%
\begin{tabular}{@{}llrrrr@{}}
\toprule
\bfseries Dataset
& \bfseries Seed
& \bfseries Val MSE
& \bfseries Val MAE
& \bfseries Test MSE
& \bfseries Test MAE \\
\midrule
\multirow{5}{*}{HumanActivity}
 & 2024 & 0.045863 & 0.118920 & 0.042599 & 0.117128 \\
 & 2025 & 0.046127 & 0.119902 & 0.042174 & 0.116331 \\
 & 2026 & 0.047570 & 0.120034 & 0.044177 & 0.117122 \\
 & 2027 & 0.046258 & 0.118295 & 0.041547 & 0.116031 \\
 & 2028 & 0.045992 & 0.119011 & 0.042301 & 0.116980 \\
\midrule
\multirow{5}{*}{USHCN}
 & 2024 & 0.353462 & 0.308676 & 0.166031 & 0.272905 \\
 & 2025 & 0.352962 & 0.305179 & 0.164978 & 0.268582 \\
 & 2026 & 0.358588 & 0.308169 & 0.164755 & 0.269127 \\
 & 2027 & 0.355586 & 0.307593 & 0.164687 & 0.273676 \\
 & 2028 & 0.353245 & 0.305297 & 0.163671 & 0.269114 \\
\midrule
\multirow{5}{*}{PhysioNet}
 & 2024 & 0.299604 & 0.354958 & 0.287338 & 0.351516 \\
 & 2025 & 0.299745 & 0.356185 & 0.287629 & 0.352951 \\
 & 2026 & 0.299682 & 0.356594 & 0.287757 & 0.353648 \\
 & 2027 & 0.299585 & 0.356553 & 0.287721 & 0.353373 \\
 & 2028 & 0.299437 & 0.354797 & 0.287664 & 0.351764 \\
\midrule
\multirow{5}{*}{MIMIC}
 & 2024 & 0.418445 & 0.390700 & 0.405562 & 0.383714 \\
 & 2025 & 0.420998 & 0.392848 & 0.404776 & 0.385185 \\
 & 2026 & 0.420267 & 0.391029 & 0.403561 & 0.383412 \\
 & 2027 & 0.419227 & 0.392369 & 0.405495 & 0.385361 \\
 & 2028 & 0.419151 & 0.391649 & 0.404404 & 0.384247 \\
\bottomrule
\end{tabular}%
}
\end{table}

\subsection{Efficiency on Matched Workloads}
\subsubsection{Measurement protocol}
We compare CTLite, APN, and GraFITi on one RTX 3090 using PyTorch 2.5.1+cu124, FP32, four CPU threads, and disabled TF32 matrix multiplication. For each dataset and seed, the same 32 training records are supplied with identical values, masks, query times, and within-batch padding. All efficiency experiments use the same five random seeds (2024–2028), with results reported as the mean and sample standard deviation across runs. APN and GraFITi retain their official architectures, optimiser settings, and native losses. In particular, GraFITi retains its variable-averaged loss and graph construction. 

Peak memory is the CUDA allocator's peak allocated memory after warm-up, including resident inputs and model state. Counters are reset before each measured operation. Training additionally includes gradients and Adam state. Allocator-reserved but not used memory and memory outside PyTorch are excluded from the displayed peak, so these values are not total device-memory occupancy. Throughput is averaged from the five individual rates, rather than calculated as the reciprocal of average latency. 

\subsubsection{Size, latency, and memory}
Table~\ref{tab:efficiency} and Fig.~\ref{fig:tradeoff} compare computational costs at a matched batch size of $B=32$. CTLite achieves the lowest parameter count, training and inference latency, and peak allocated memory, together with the highest throughput, on all four datasets. 

\begin{figure}[t]
 \centering
\includegraphics[width=\linewidth]{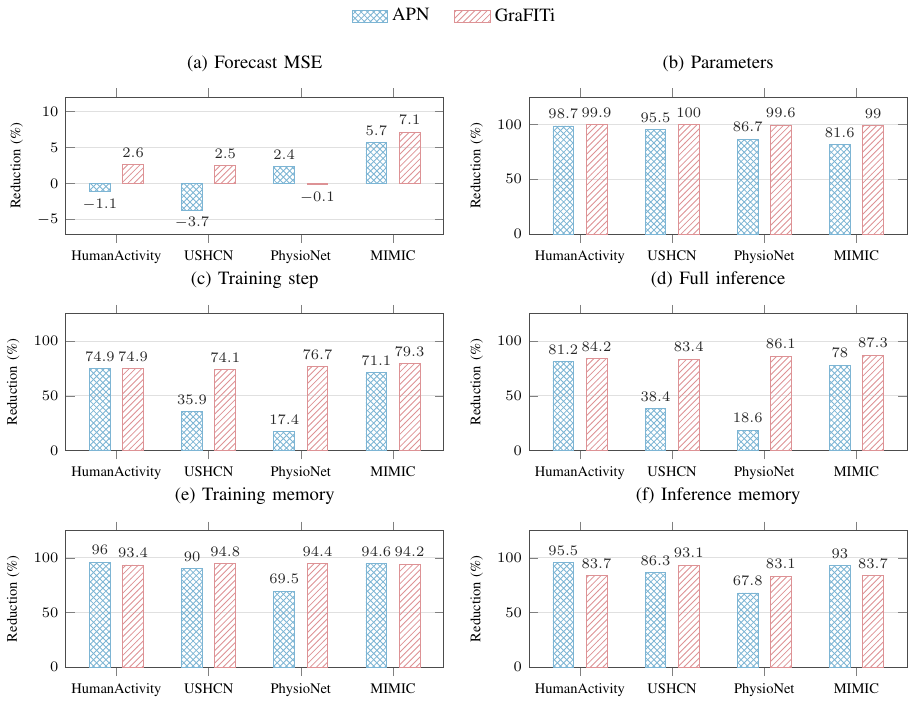}
\caption{Relative reductions, $100(1-m_{\mathrm{CTLite}}/m_{\mathrm{base}})$. Positive values favour CTLite.}
 \label{fig:tradeoff}
\end{figure}

CTLite contains only 88--2,537 parameters, requiring 352--10,148 bytes of FP32 parameter storage.
Relative to APN and GraFITi, the parameter count is reduced by 81.59--98.73\% and 99.00--99.98\%, respectively.

The compact components also provide substantial runtime benefits. CTLite completes full inference in
0.654--0.667~ms, reducing latency by 18.60--81.16\% compared with APN and 83.40--87.27\% compared with GraFITi. Its throughput reaches 47.98--48.93 thousand records per second, corresponding to 1.23--5.31$\times$ that of APN and 6.02--7.60$\times$ that of GraFITi. For each run, throughput is computed as $1000B/t$, where $t$ is inference latency in milliseconds. Training steps require 2.205--2.238~ms, yielding reductions of 17.42--79.29\% across the two baselines. The improvement over APN is smallest on PhysioNet, where its inference and training costs are already relatively low. Thus, the magnitude of the runtime benefit depends on both the dataset and the baseline, despite consistently lower costs for CTLite under this workload.

CTLite also reduces the memory required for execution. Peak allocated memory is 17.70--20.99~MiB during training and 9.55--12.78~MiB during inference, representing reductions of 69.50--96.02\% and 67.80--95.54\%, respectively, across the two baselines. Here, one MiB equals $2^{20}$ bytes. Unlike parameter storage, these measurements include resident inputs and intermediate tensors, as well as gradients and optimiser state during training. The memory savings therefore extend beyond the model weights to the measured execution footprint. All primary measurements include historical encoding without using cached features.

\subsubsection{CPU-only deployment}
We evaluate single-record inference ($B=1$) on an Intel Core i7-12700KF using one and four CPU threads. All models process identical requests in FP32, with full history encoding and no feature caching. Inputs reside in CPU memory, and timing excludes data loading and process startup. We compute median and 95th-percentile (p95) request latency separately for each seed and report their means and sample standard deviations over five seeds.

\begin{figure}[t]
 \centering
\includegraphics[width=\linewidth]{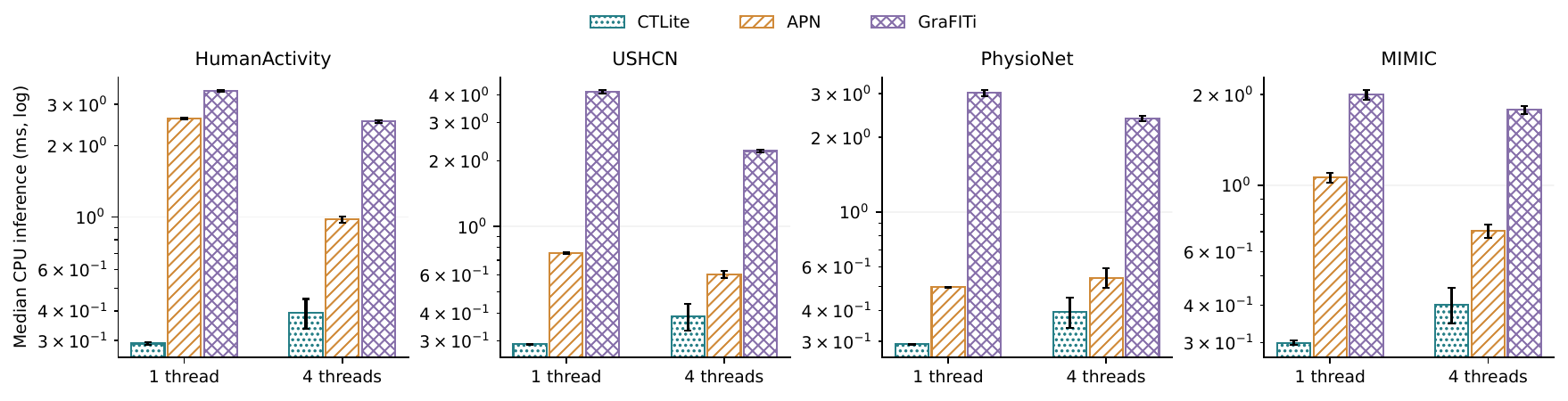}
\caption{ Full per-request CPU inference without cached histories. Bars show five-seed mean median latency, and sample SD. The log axis accommodates model differences. Input tensors are CPU resident; disk IO and process startup are excluded.}
 \label{fig:cpu_deployment}
\end{figure}

CTLite achieves the lowest mean median latency in all eight dataset--thread settings. With one thread, its latency is $0.291\pm0.003$, $0.289\pm0.002$, $0.292\pm0.001$, and $0.300\pm0.004$~ms on HumanActivity, USHCN, PhysioNet, and MIMIC, respectively. Across both thread settings, CTLite reduces mean median latency by 27.03--88.86\% relative to APN and 77.44--93.01\% relative to GraFITi. It also achieves the lowest mean p95 latency in every setting, indicating that the advantage extends beyond typical requests to the latency tail.

Increasing the thread count does not improve CTLite's latency, suggesting limited benefit from additional parallelism for its small computational workload. Nevertheless, CTLite remains faster than both baselines with four threads. These results complement the GPU measurements by demonstrating efficient single-record execution without accelerator hardware or cached summaries.

\subsection{Batch Scaling}
\label{sec:batch_scaling}
We repeat the complete benchmark at $B\in\{1,16,32,64\}$ using the final model settings. Each seed selects 64 training records per dataset. All three models receive the same values, masks, times, and padding at a given batch size. Each model--dataset--seed--batch combination uses a fresh GPU process, yielding 240 isolated workloads. This prevents allocations and optimiser state from carrying across batch sizes. Table~\ref{tab:efficiency} and Fig.~\ref{fig:tradeoff} use the $B=32$ slice of this same measurement set.

\begin{figure*}[t]
 \centering
\includegraphics[width=\linewidth]{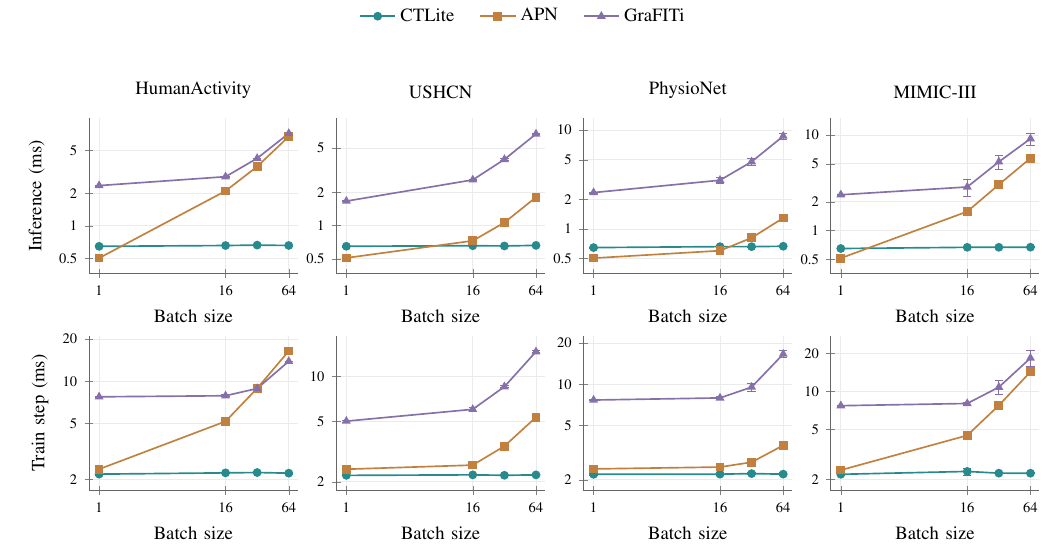}
\caption{Batch scaling. Full inference above and full training steps below. Error bars show sample SD over five seeds.}
\label{fig:batch_scaling}
\end{figure*}

Fig.~\ref{fig:batch_scaling} reports full inference and full training steps, including historical encoding. Both axes use logarithmic scales, and error bars show the sample standard deviation of five run medians. CTLite's mean inference cost grows from 0.648--0.650 ms at $B=1$ to 0.660--0.669 ms at $B=64$. Its training-step cost remains within 2.174--2.310 ms over all measured datasets and batch sizes. The shallow, shared readout therefore accommodates more records without a proportional increase in elapsed GPU time on these workloads.

Especially, as the batch size increases from 1 to 64, CTLite's inference and training-step times remain nearly constant across all four datasets, whereas both baselines exhibit substantial increases, particularly at larger batches. CTLite therefore processes more records per step with little additional latency, translating larger batches into higher throughput. This behaviour benefits both batched prediction and training, where the computational cost per record decreases as more records are processed together. The runtime advantage becomes especially clear at batch sizes of 32 and 64, where CTLite is faster than both baselines on every dataset. Although APN achieves lower inference latency at batch size 1 and on PhysioNet at batch size 16, CTLite provides more favourable scaling over the evaluated batch sizes. These results demonstrate CTLite's efficiency for batched workloads under the same tested hardware and sequence lengths.

\begin{table}[!t]
\centering
\caption{Historical encoding and computation with a stored summary (ms).}
\label{tab:cache}
\small
\setlength{\tabcolsep}{2pt}
\renewcommand{\arraystretch}{1.14}

\begin{tabular*}{\columnwidth}{@{\extracolsep{\fill}}l!{\color{black!40}\vrule width .35pt}ccc@{}}
\toprule
\bfseries Dataset & \bfseries Encode once & \bfseries Cached inference & \bfseries Cached train step\\
\midrule
HumanActivity & 0.370$\pm$0.004 & 0.336$\pm$0.004 & 1.817$\pm$0.024\\
USHCN & 0.371$\pm$0.003 & 0.331$\pm$0.004 & 1.786$\pm$0.006\\
PhysioNet & 0.376$\pm$0.008 & 0.341$\pm$0.009 & 1.801$\pm$0.038\\
MIMIC & 0.373$\pm$0.005 & 0.331$\pm$0.003 & 1.806$\pm$0.031\\
\bottomrule
\end{tabular*}

\end{table}

\begin{table*}[!t]
\centering
\caption{Component ablations. Mean $\pm$ sample SD over five seeds.}
\label{tab:component_ablation}

\scriptsize
\setlength{\tabcolsep}{1pt}
\renewcommand{\arraystretch}{1.1}

\resizebox{\textwidth}{!}{%
\begin{tabular}{@{}l|cc|cc|cc|cc@{}}
\hline
\textbf{Dataset}
& \multicolumn{2}{c|}{\textbf{HumanActivity}}
& \multicolumn{2}{c|}{\textbf{USHCN}}
& \multicolumn{2}{c|}{\textbf{PhysioNet}}
& \multicolumn{2}{c}{\textbf{MIMIC}} \\
\hline
\textbf{Variant}
& MSE $\downarrow$ & MAE $\downarrow$
& MSE $\downarrow$ & MAE $\downarrow$
& MSE $\downarrow$ & MAE $\downarrow$
& MSE $\downarrow$ & MAE $\downarrow$ \\
\hline
CTLite
& $\mathbf{0.04256\pm0.00098}$
& $\mathbf{0.11672\pm0.00051}$
& $0.16482\pm0.00084$
& $0.27068\pm0.00241$
& $\mathbf{0.28762\pm0.00017}$
& $\mathbf{0.35265\pm0.00096}$
& $\mathbf{0.40476\pm0.00083}$
& $\mathbf{0.38438\pm0.00087}$ \\
\hline
w/o (b)
& $0.05438\pm0.00047$
& $0.12939\pm0.00068$
& $0.17590\pm0.00147$
& $0.28910\pm0.00150$
& $0.32980\pm0.00011$
& $0.37902\pm0.00064$
& $0.41306\pm0.00056$
& $0.38969\pm0.00034$ \\
w/o (c)
& $0.05738\pm0.00025$
& $0.12646\pm0.00025$
& $0.30717\pm0.00302$
& $0.32089\pm0.00446$
& $0.34859\pm0.00062$
& $0.38170\pm0.00033$
& $0.48643\pm0.00191$
& $0.42079\pm0.00132$ \\
w/o (d)
& $0.04259\pm0.00080$
& $\mathbf{0.11672\pm0.00131}$
& $\mathbf{0.16475\pm0.00093}$
& $\mathbf{0.26739\pm0.00097}$
& $0.28982\pm0.00008$
& $0.35309\pm0.00091$
& $0.42632\pm0.00062$
& $0.39337\pm0.00052$ \\
\hline
\end{tabular}%
}
\end{table*}

\subsection{Efficiency Gains from Reusing Historical Summaries}

CTLite separates historical encoding from prediction. The encoder converts observed values, timestamps, and masks into a fixed summary containing temporal moments and observation-state features. The prediction modules then combine this summary with the requested forecast times. Because the encoder contains no learned parameters, an unchanged history produces the same summary throughout training. The summary can therefore be computed once and reused across subsequent optimiser steps or repeated forecast queries, while the learned prediction modules are evaluated each time.

Table~\ref{tab:cache} measures the computational benefit of this reuse at a batch size of 32. \emph{Encode once} reports the cost of constructing the summary. \emph{Cached inference} reports inference time when the summary is already available in GPU memory. \emph{Cached train step} reports a training step using the stored summary, including prediction, loss computation, backpropagation, and the optimiser update. The latter two measurements exclude summary construction, which is performed before reuse begins.

Across the four datasets, summary construction takes 0.370--0.376~ms. Once the summary is available, inference requires only 0.331--0.341~ms, reducing latency by 48.47--50.40\% relative to CTLite's full inference in Table~\ref{tab:efficiency}. Training steps require 1.786--1.817~ms, a reduction of 18.58--19.33\%. These savings come from avoiding repeated historical encoding, without changing the forecasting computation. We verify exact equality between predictions obtained from stored summaries and complete forward execution in the measured implementation.

Reuse is beneficial when the same historical window is accessed repeatedly. A new observation, a shifted window, or a change in preprocessing or encoder settings requires recomputing the summary. Updating the learned prediction weights alone does not. The reported cached timings assume GPU-resident summaries
and exclude persistent-storage and retrieval overhead. \textit{All primary efficiency comparisons in Table~\ref{tab:efficiency} and Fig.~\ref{fig:batch_scaling} include historical encoding.}

\subsection{Component Analysis}
\label{sec:ablation}
We evaluate the three main modules in CTLite by removing each module separately and retraining the resulting variant. All experiments use the final dataset-specific configurations and five random seeds (2024--2028). The training objective, learning-rate schedule, and validation-based checkpoint selection remain unchanged, without separate tuning of the variants. Table~\ref{tab:component_ablation} reports the mean and sample standard deviation.

\textit{Ablation design.}
Removing (b) bypasses the continuous-time moment bank. To preserve the interface to the remaining modules, we replace each temporal mean $\mu_{kn}$ with the most recent observation $\ell_n$ and set each slope $b_{kn}$ to zero, where $k$ indexes temporal scales and $n$ indexes variables. Observation age, count, availability, and the time origin remain available. Removing (c) eliminates the entire local prediction branch, while removing (d) eliminates the entire cross-variable correction branch.

\textit{Temporal encoding and local prediction.}
Removing (b) increases mean MSE by 27.77\%, 6.72\%, 14.67\%, and 2.05\% on HumanActivity, USHCN, PhysioNet, and MIMIC, respectively. Removing (c) causes larger increases of 34.82\%, 86.37\%, 21.20\%, and 20.18\%. Both removals increase MSE and MAE in every paired seed across all four datasets. These results support the complementary roles of the moment encoder in summarising historical evolution and the local branch in translating this summary into predictions at the requested times.

\textit{Cross-variable correction.}
Removing (d) increases mean MSE from 0.28762 to 0.28982 on PhysioNet and from 0.40476 to 0.42632 on MIMIC, corresponding to increases of 0.76\% and 5.33\%. Both increases occur in all five seeds. Its effect on HumanActivity is small, while its removal slightly improves USHCN. Thus, the benefit of cross-variable correction depends on the dataset and is most evident on the two clinical benchmarks.

\section{Conclusion}
We presented CTLite, a compact framework for irregular multivariate time series forecasting that separates fixed historical summarisation from learned prediction. Continuous-time moments retain local levels and supported trends, while a small readout combines damped temporal corrections with observation-conditioned cross-variable information. This separation makes historical summaries reusable across optimisation steps and forecast queries without changing predictions or parameter gradients. Across four benchmarks, CTLite achieves the best or second-best mean MSE and MAE among the evaluated methods with only 88–2,537 trainable parameters, alongside lower latency and memory consumption than the latest models at a matched batch size of 32. These results show that a carefully structured historical representation can support accurate irregular forecasting with a small learned predictor and make repeated historical computation avoidable. 
\label{main:lastpage}


\bibliographystyle{IEEEtran}
\bibliography{IEEEabrv,references}

\end{document}